\documentclass[letterpaper, 10 pt, conference]{ieeeconf}  

\IEEEoverridecommandlockouts                              

\usepackage{graphicx} 
\usepackage{amsmath} 
\usepackage{amssymb}  
\usepackage{mathtools}
\usepackage{bm}
\usepackage{color}
\usepackage{colortbl}
\usepackage{algorithm}
\usepackage{svg}
\usepackage[noend]{algpseudocode}
\usepackage{pbox}
\usepackage{hyperref}

\DeclareMathOperator*{\argmin}{argmin}

\newcommand{\norm}[1]{\left\lVert#1\right\rVert}
\newcommand{\mat}[0]{\begin{bmatrix}}
\newcommand{\mate}[0]{\end{bmatrix}}
\newcommand{\tn}[1]{\textnormal{#1}}

\newcommand{\vu}{\mathbf{u}}

\newcommand{\vx}{\mathbf{x}}

\usepackage[labelformat=simple]{subcaption}
\DeclareCaptionLabelSeparator{periodspace}{.\quad}
\usepackage{amsthm}

\theoremstyle{plain}

\newtheorem{prop}{Proposition}

\newtheorem{define}{Definition}

\newtheorem{problem}{Problem}

\renewcommand{\baselinestretch}{0.95}

\title{\LARGE \bf
Multi-robot Task Assignment for Aerial Tracking \\ with Viewpoint Constraints
}

\author{Aaron Ray$^{1}$, Alyssa Pierson$^{1}$, Hai Zhu$^{2}$, Javier Alonso-Mora$^{2}$, Daniela Rus$^{1}$
\thanks{$^{1}$Computer Science and Artificial Intelligence Laboratory, Massachusetts Institute of Technology, Cambridge, MA 02139, USA
        {\tt\small \{aray, apierson, rus\}@csail.mit.edu}}%
\thanks{$^{2}$Cognitive Robotics, Delft University of Technology, 2628 CD, Delft, Netherlands
        {\tt\small \{H.Zhu, J.AlonsoMora\}@tudelft.nl }}%
\thanks{This work supported in part by the Office of Naval Research (ONR), and Singapore's DSO National Laboratory and SUTD under the ASTRALIS project.
Their support is gratefully acknowledged. }%
}

\begin{document}

\maketitle
\thispagestyle{empty}
\pagestyle{empty}

\begin{abstract}

We address the problem of assigning a team of drones to autonomously capture a set desired shots of a dynamic target in the presence of obstacles. We present a two-stage planning pipeline that generates offline an assignment of drone to shots and locally optimizes online the viewpoint. Given desired shot parameters, the high-level planner uses a visibility heuristic to predict good times for capturing each shot and uses an Integer Linear Program to compute drone assignments. An online Model Predictive Control algorithm uses the assignments as reference to capture the shots. The algorithm is validated in hardware with a pair of drones and a remote controlled car.

\end{abstract}

\section{INTRODUCTION}

We wish to develop algorithms for coordinating heterogeneous systems of aerial and ground agents when the ground agents are not cooperating with the aerial vehicles. One class of problems within this broad scope is following a ground-based moving agent (e.g. robot or human) with an aerial vehicle, subject to constraints such as “keep a certain feature of the agent in the field of view while avoiding environmental obstacles”. This problem is challenging because it requires a real time adaptive solution for the local control with global objectives and constraints. Practical and robust solutions will enable new applications such as autonomous drone videography that go beyond today’s recording capabilities. Current drone videography can follow an actor. In this paper we describe a solution that supports finer-grain specifications, such as “keep the actor’s face in the field of view”. The solution has to combine real-time local response with global planning to accommodate the presence of obstacles (e.g. avoid bridges). 

More specifically, we enable a team of videography drones to autonomously track and capture a sequence of desired shots of a moving target such as a ground robot or a person. Framing a scene for videography is a complicated process that depends on a range of aesthetic
preferences of the videographer. However, many of the important framing primitives can be distilled into a small set of parameters that define the camera's desired viewpoint, such as size of the target in the frame, position of the target in the image, and position of the camera relative to the target. We would like for a videographer to be able to specify a set of desired shots based on these parameters, and have the team of drones determine the best trajectories for capturing the video of the subject as it moves through the environment. We assume access to a good prediction of the subject's motion, as otherwise there is little use in pre-planning long sequences of shots.

\begin{figure}
    \centering
    \includegraphics[width=.8\columnwidth]{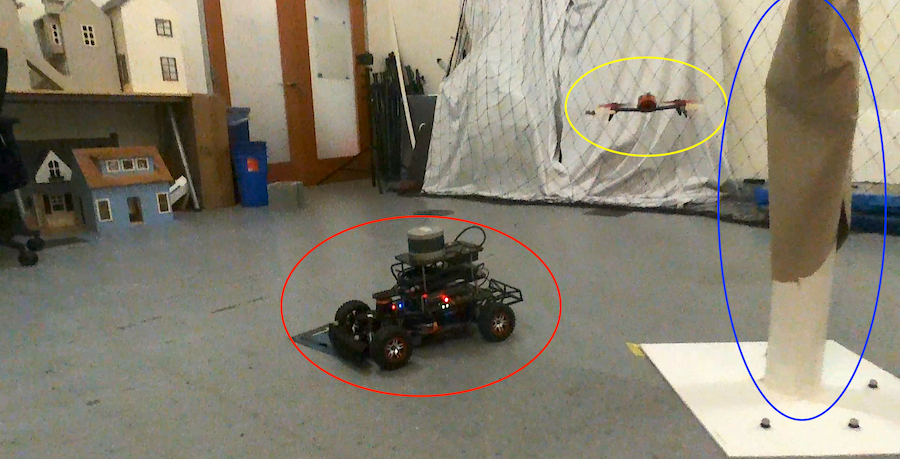}

    \caption{Viewpoint from a videography drone looking at a target (red). Another drone (yellow) optimizes the capture of a different shot while avoiding occlusions from obstacles (blue).}
    \label{fig:on_drone}
    \vspace{-.6cm}
\end{figure}

We present a two-stage viewpoint optimization pipeline that enables a team of drones to capture a series of shots that match desired aesthetic qualities. A high-level planner uses a visibility heuristic and an Integer Linear Program (ILP) optimization routine to choose when each drone should capture which shot. This assignment results in a reference trajectory that can be tracked by an online Model Predictive Control (MPC) algorithm with a cost function based on the specified viewpoint parameters. The controller can locally optimize the drones' trajectories to account for stochastic target motion. We demonstrate a videography scenario with a pair of
drones assigned to capture several shots of a remote controlled racecar in the presence of ellipsoidal
obstacles. 

\textbf{Related work} A large body of literature exists related to defining shot aesthetics for a videography drone \cite{mademlis_autonomous_2019-1}. These works explore parameterizing desired shot qualities and controls \cite{galvane_directing_2018}\cite{christie2008}\cite{gleicher1992}\cite{lino_intuitive_2015}, feasibility of dynamic shots \cite{karakostas_shot_2019}, and assigning sequences of shots \cite{mademlis_high-level_2019}\cite{Drucker94intelligentcamera}. Other work as focused on algorithmic frameworks for helping directors achieve desirable aesthetic qualities of their shots\cite{lino2011}\cite{jiang2020}.

Drone videography requires precise motion planning and control algorithms to achieve the defined aesthetic objectives. One common approach solves a constrained nonlinear optimization in a receding horizon fashion \cite{sabetghadam_optimal_2019, nageli_real-time_2017, nageli_real-time_2017-1}. Other approaches optimize for trajectory smoothness \cite{gebhardt_optimizing_2018}, focus on a series of static landmarks \cite{xie_creating_2018}, or use deep reinforcement learning \cite{bonatti_autonomous_nodate}. More generally, the problem of tracking multiple subjects is similar to persistent monitoring \cite{Alamdari2014, Tokekar2015};  patrolling and surveillance \cite{Pippin2013, Robin2016, Khan2018, deMoraes2020}; and pursuer-evader games \cite{Vieira2009, PiersonEtAl2017_IRaAL, ZhouEtAl2020_SJoCaO} .

\textbf{Contributions} We build upon the authors' previous work in \cite{nageli_real-time_2017} and focus on the problem of optimizing sequences of shots from multiple cameras and perspectives subject to constraints. In contrast to \cite{nageli_real-time_2017} and other recent work \cite{bonatti_autonomous_nodate}\cite{bucker2020} that has focused on reactively finding good viewpoints in unstructured scenes, we assume that the operational environment is known and the subject's motion can be estimated ahead of time. In this context, a director may have a set of desired shots that should be captured during the scene. It is nontrivial to decide when each shot should be taken so that the shots are minimally obstructed and the ordering is feasible for the drones. A higher-level global planner is used to augment the online local planner and ensure the desired viewpoints can all be captured. The main contributions of the paper are:
\begin{itemize}
\item Presenting a novel high-level videography planner based on a visibility heuristic to globally optimize ordering of shots of a dynamic target
\item Showing that the reference trajectory generated by this assignment can be followed by a viewpoint-aware receding horizon controller to locally optimize shot aesthetics
\item Demonstrating these algorithms in hardware experiments with a pair of drones and remote control racecar videography target
\end{itemize}
The remainder of this paper is organized as follows: we briefly summarize shot aesthetic specifications and formalize the minimization of
the viewpoint costs in Section \ref{sec:vid}. Section \ref{sec:assignment} presents the high-level
global planner which generates a sequence of shots and reference trajectory to guide each drone. In
Section \ref{sec:mpc} we explain the videography MPC formulation used for local viewpoint optimization. Section
\ref{sec:experiments} describes our experimental results.

\section{Problem Definition}\label{sec:vid}
When a predictable subject is operating in a known environment, we would like to be able to coordinate a set of desired shots such that each shot is captured with minimal occlusion. For example, repeated takes on a movie set or a well-trodden mountain bike path are scenarios where the subject's motion is fairly restricted and can be estimated ahead of time, even if the exact path is not known. Existing videography approaches focus on reactive planning to maintain good views of a target. We are more interested in global planning that ensures a team of videography drones sequences the set of desired shots to ensure they are all fit in within the allotted time and with minimal obstruction.
\subsection{Preliminaries}

Throughout this paper vectors are denoted in bold lowercase letters, $\mathbf{w}$, matrices in plain
uppercase, $M$, and sets in calligraphic, $\mathcal{S}$. $\norm{\mathbf{w}}$ denotes the Euclidean norm of
$\mathbf{x}$ and $\norm{\mathbf{w}}_{Q}^2 = \mathbf{w}^TQ\mathbf{w}$ denotes the weighted squared norm. We use matrices $Q_{\_\_}$ that appear in these weighted norms are used to represent tuneable parameters in cost functions that allow the user to express the relative importance of different desiderata. 

We assume that $n_d$ videography drones operate in an environment with known set of obstacles $\mathcal{O}$. While the MPC formulation we use requires obstacles to be represented as ellipsoids, our main contribution of the higher-level planner supports any obstacle set representation that allows for computing ray intersections efficiently. 

We also assume access to a probability distribution over the videography subject's future trajectory, denoted $\mathbb{Y}$. We sample from this distribution to predict good times for shots to be taken.

\subsection{Defining Desired Shots}

Much of the defining aesthetic of a camera framing is determined by:

\begin{enumerate}
    \item \textbf{Orientation} of the camera relative to the film subject, denoted by azimuth and elevation, $(\psi, \theta)$,
    \item \textbf{Subject's Distance} to the camera, $\rho$,
    \item \textbf{Subject's Position} in the frame, $(C_x, C_y)$.
\end{enumerate}
as presented in \cite{nageli_real-time_2017}. We also note that it is usually desirable for the videography subject not to be occluded in the image.

In addition to the spatial viewpoint parameters, we consider a time window within which a shot should be
taken, $[t_0, t_f]$, and a desired duration $\tau^i$. 
\begin{define}[Shot Specification]
A shot specification $S_i = (\theta^i, \psi^i, \rho^i, C_x^i, C_y^i, t_0^i, t_f^i, \tau^i)$ is set of viewpoint parameters and timing constraints define each shot. We denote the set of desired shot specifications $\mathcal{S}$. 
\end{define}
\begin{figure}
\centering
    \includegraphics[width=.7\columnwidth]{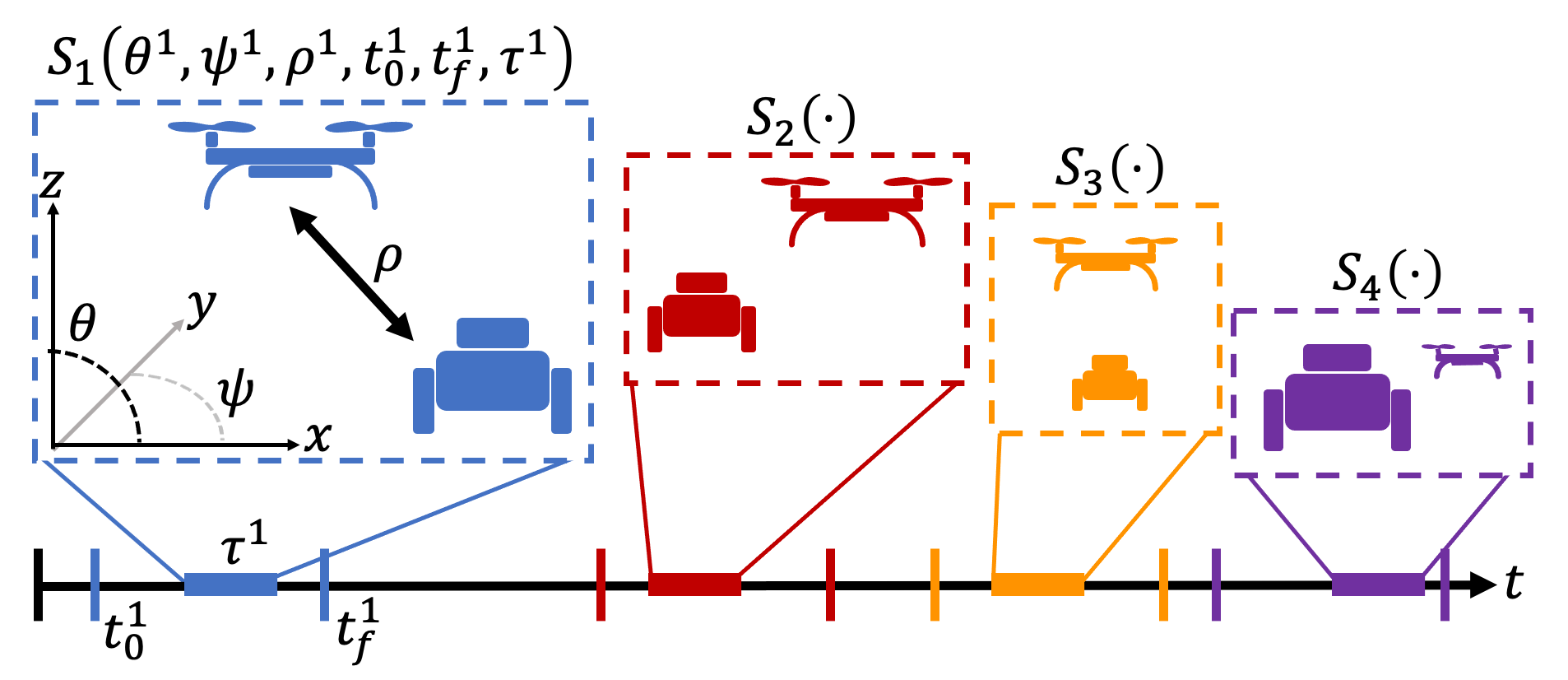}
    \caption{\label{fig:shot_schematic}Illustration of the shot schematic. Within each shot, we optimize based on viewpoint costs. We also optimize the shot assignment for the sequence of shots between drones.}
    \vspace{-.6cm}
\end{figure}
Figure \ref{fig:shot_schematic} illustrates these parameters. The viewpoint properties depend on the drone state $\mathbf{x}_b$, target state
$\mathbf{x}_t$, and obstacles in the environment $\mathcal{O}$. We briefly define a cost function associated with each parameter, and refer to \cite{nageli_real-time_2017} for derivation. 
Let $\mathbf{r}_{ct}$ and $\mathbf{r}_{ct}^c$ denote the vector from the camera origin to the target expressed in the world and camera reference frames respectively. Let $\mathbf{r}_d^c$ denote the desired vector from the camera origin to target as defined by image position parameters $C_x$ and $C_y$. Let $\bm{\alpha}_d$ represent the unit vector corresponding to the desired angles $(\theta, \psi)$ for the drone relative to the target. With these definitions, we can define cost functions that penalize a drone's state as it deviates from the desired viewpoint.
Let $c_{image}$ denote the cost for the target's position in the image:
\begin{equation*}
    c_{image} = \norm{\frac{\mathbf r_d^c}{\norm{\mathbf r_d^c}} - \hat{\mathbf r}_{ct}^c}_{Q_i}.
\end{equation*}
The cost $c_{scale}$ penalizes distance deviation from the target,
\begin{equation*}
    c_{scale} = \norm{\norm{\mathbf r_{ct}} - \rho}_{Q_s},
\end{equation*}
and $c_{angle}$ is a cost that depends on the deviation of the drone's angles relative to the target:
\begin{equation*}
    c_{angle} = \norm{-\frac{\mathbf r_{ct}}{\norm{\mathbf r_{ct}}} - \frac{\mathbf
\alpha_d}{\norm{\mathbf \alpha_d}}}_{Q_a}.
\end{equation*}

Finally, $c_{occlusion}$ penalizes drone positions that have an obstructed view of the target. This cost is based on the position of the target compared to the occlusion cone associated with the camera and each obstacle. We define a cost based on $d_v$, the distance by which the target is within an obstacle's occlusion cone. Let $\mathbf r_{co}^o$ represent the vector from the camera to an obstacle in the obstacle's reference frame. We define $c_{occlusion} = \max(0, d_v)^2$ where
\begin{equation*}
    d_v = \frac{\mathbf r_{co}^{o^T}\mathbf r_{ct}^o}{\sqrt{\norm{\mathbf r_{co}^o}^4 - \norm{\mathbf r_{co}^o}^2}} - \sqrt{\norm{\mathbf r_{ct}^o}^2 -
\frac{(\mathbf r_{co}^{o^T}\mathbf r_{ct}^o)^2}{\norm{\mathbf{r}_{ct}^o}^2}},
\end{equation*}
if $\mathbf r_{ct}^{o^T}\mathbf r_{co}^o/\norm{\mathbf{r}_{co}^o} > \norm{\mathbf r_{co}^o} - 1$.  Otherwise, the target is closer to the camera than the obstacle is and $c_{occlusion} = 0$. Each obstacle contributes a separate occlusion cost. A similar visibility cone concept can be used to penalize mutual visibility between drones as in \cite{nageli_real-time_2017-1}, although we do not employ that cost here.

\subsection{Defining Desired Sequences}

These costs $c$ aid in quantifying the quality of a single shot. Our viewpoint objective is to minimize these costs over a sequence of shots.

\begin{define}[Shot Sequence]\label{def:shot-sequence}
A shot sequence $A_i = (S_j^{t_a}, S_k^{t_b}, \hdots)$ is the set of all shots assigned to drone $i$. Each assigned shot $S_j^{t_a}$ corresponds to a shot specification $S_j$ and a time when the drone should be capturing that shot, $t_a$. $A_i^j$ denotes the $j$-th shot assigned to drone $i$. The set of all assigned shots across all drones and times is $\mathcal{A} = \cup_i A_i$.
\end{define}

For each shot, we define a vector of cost functions $\mathbf{J}_i^{j}$ that
penalizes drone $i$'s state from deviating from the parameters defined by shot $S_j$.
$\mathbf{J}_i^{j}$ contains penalties for deviating from the desired location of the target in the
image plane, its size in the image plane, and the camera position relative to the target.
$\mathbf{J}_i^{j}$ also penalizes drone states that have an obstructed view of the target. Intuitively, 
\begin{equation}
    \mathbf{J}_i^{j} = [c^i_{\tn{image}}, c^i_{\tn{scale}}, c^i_{\tn{angle}}, c^i_{\tn{occlusion}}]^T.
\end{equation}
Given the trajectory of the videography target, 
 $\mathbf{J}_{i}^{j}$ is a function of drone
$i$'s state and time. For a given set of shots $\mathcal{S}$, our goal is to then find the trajectories for all drones such that  $\mathbf{J}$ is minimized over all shots. 
\begin{problem}
For a desired set of shots $\mathcal{S}$ and a target trajectory $Y$, we define the total
videography cost $L$ as a function of the trajectories of the $n_d$ drones:
\begin{equation} \label{eq:optimization-full}
\begin{split}
    L(\mathbf{x}_1, \hdots, \mathbf{x}_{n_d})  = \sum_{j=0}^{|S|} \min_{t_0^j \leq t \leq t_f^j} \min_{i} \int_{t}^{t + \tau^j}
 \norm{\mathbf{J}_i^j(\mathbf{x}_i(t), t)}_{Q_x} dt.\\
\end{split}
\end{equation}
In general, $Y$ is stochastic, and by extension $\mathbf{J}^j_i$ is as well.
We seek a control policy for each drone that will minimize the expected videography cost $L$.
\end{problem}
We use a two stage approach to find good policies for minimizing $L$. First, a high-level planning
algorithm uses a simple heuristic to predict good times for capturing each shot. A discrete
optimization algorithm assigns each drone to a sequence of shots that minimizes the heuristic cost.
Next, the optimized assignment and expected target trajectory are used to generate a reference
trajectory for a videographic Model Predictive Control (MPC) algorithm that locally optimizes the
aesthetic parameters online.
\section{SHOT ASSIGNMENT}\label{sec:assignment}
The high-level shot planning is carried out by a centralized algorithm that finds a sequence of shots for each drone to capture, as well
as a reference trajectory to follow. The algorithm requires the desired set of shots $\mathcal{S}$,
a distribution over target trajectories $\mathbb{Y}$, and a set of obstacles in the environment
$\mathcal{O}$. It relies on a heuristic to choose a low-cost reference trajectory that guarantees
visibility of the target. The heuristic provides an estimate of the best times to start capturing
each shot. We sample the lowest-cost times for each shot and assign each drone a sequence of shots and times based on an ILP minimization.

\subsection{Assignment Heuristic}\label{sec:assignment-heuristic}

We seek a simple heuristic for each shot $i$ that maps time to a reference
position, while ensuring that the target is visible. We refer to the position given by this heuristic as $\hat{\mathbf{x}}^i(t)$, and $H^i(t)$ as the associated cost. By focusing on a single cost and position at each time, the problem of optimizing shot cost over all possible trajectories simplifies to choosing a sequence of shots and times for each drone. 

The target is
visible from the drone if and only if a ray cast from the target to the drone hits no obstacles
before reaching the drone. This observation inspires a heuristic for the reference shot position ---
cast a ray from the target's position $\mathbf{x}_t$ in the desired shot direction, stopping at the desired shot distance or
the intersection with an obstacle, whichever is first. Let $\vx^*(t) = \vx_t(t) +
\rho^i\bm{\alpha}$ and $\hat{\vx}(t) = \vx_t(t) + d\bm{\alpha}$
where
\begin{equation*}
\bm{\alpha} = [\cos\theta\cos(\psi + \psi_t), \cos\theta\sin(\psi + \psi_t), \sin\theta]^T,
\end{equation*}
and $d$ is the smaller of $\rho_i$ and the closest obstacle intersection along $\bm{\alpha}$. Intuitively, $\mathbf{x}^*$ is the drone position that would have zero viewpoint cost, and $\hat{\mathbf{x}}$ is as close as we can get to that position along the desired direction from the target, before hitting an obstacle. We similarly define $\tilde{\vx}(t)$ as a ray cast from $\mathbb{E}(\vx(t))$ in the expected direction of $\bm{\alpha}$. The path $\tilde{\mathbf{x}}$ will be used as the reference path to transition between shots.
As the reference position is constructed such that it satisfies the desired relative angles $\psi$ and $\theta$, $c_{scale}$ is the only nonzero viewpoint cost.  The reference trajectory is also guaranteed to have visibility of the target.
For shot $i$, we consider a heuristic videography cost $H_{vid}^i$ that is a function of only time:
\begin{equation*}
    H_{vid}^i(t) = \norm{\mathbf{x}^*(t) - \hat{\mathbf{x}}(t)}.
\end{equation*}
The position heuristic $\hat{\vx}$ does not enforce continuity of the trajectory, so we add an additional discontinuity
cost $H_{dis}$ for each shot:
\begin{equation*}
    H_{dis}^i = q_{dis}\norm{\frac{d\hat{\mathbf{x}}}{dt}}^2,
\end{equation*}
where $q_{dis}$ is a tuning parameter that adjusts the importance of ensuring the reference trajectory can be perfectly tracked.

Together, $H_{vid}$ and $H_{dis}$ sum to the heuristic for instantaneous cost, $H_{fine}$. Each shot
$S_i$ lasts for a duration $\tau^i$. We define a shot cost $H_{shot}^i$ that estimates
the cost incurred by capturing shot $S_i$ starting at time $t$ by setting $H_{shot}^i(t)$ to be the average value of $H_{fine}$ in the interval $[t, t + \tau_i]$.  
We take the expectation over the target's trajectory distribution and calculate it by sampling. 

We can now use this heuristic $H_{shot}$ to find good times to capture each shot by
finding an assignment of shots to drones such that the sum of $H_{shot}$ over all of the chosen shot
times is minimized and each drone can be expected to feasibly capture all of its assigned shots. The size of the optimization problem depends on how finely we sample $H_{shot}$. We reduce the number of times considered for starting each shot by
sampling $p$ random time indices $\mathbf{t}_{sampled}^i$ for each shot, biasing the samples toward lower-cost times. The sampling process greatly reduces the size of the problem.
\vspace{-.2cm}
\subsection{Choosing the Best Shots}

\begin{figure}
    \begin{subfigure}{\columnwidth}
    \centering
    \includegraphics[width=.65\columnwidth]{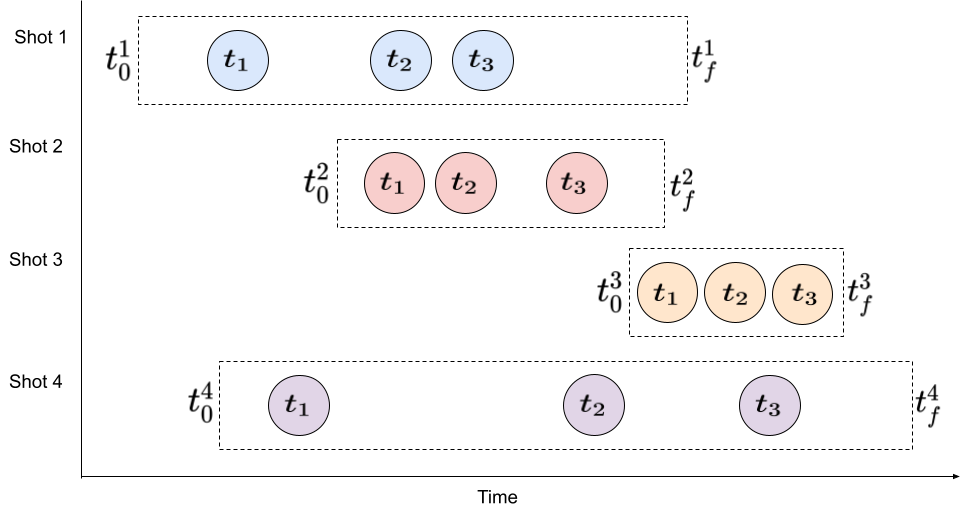}
    \caption{An example of shots four shots, $S_1,~S_2,~S_3,~S_4$. Three times with low heuristic cost $H_{shot}$ have been sampled for each shot, as discussed in Section \ref{sec:assignment-heuristic}. In reality, many more times are sampled. Our experiments use 20 samples per shot.}
    \end{subfigure}
    \begin{subfigure}{\columnwidth}
    \centering
    \includegraphics[width=.65\columnwidth]{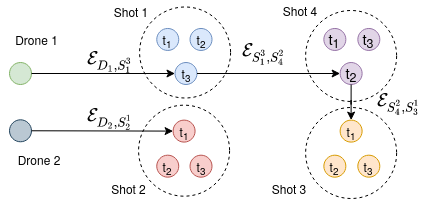}
    \caption{\label{fig:optimization_constraints} An example of a solution to the PDAG-Minimum-Paths problem for $m = 2$, where the edges not on a solution path have been omitted. Note that each color set is on exactly one path, and there are two paths. This diagram also exemplifies how a minimization in the form of \eqref{eq:assignment_min} maps to an instance of PDAG-Minimum-Paths. The edge labels $\mathcal{E}$ demonstrate the semantics of the construction of $\mathcal{G}_{shot}$. Note that the three times chosen for Drone 1's assignment are in ascending order in (a).}
    \end{subfigure}
    \caption{An illustrative example of how (a) a set of desired shots are considered at discrete times and (b) turned into a graphical optimization problem.} 
    \label{fig:optimization_constraints_full}
    \vspace{-.7cm}
\end{figure}

An assignment of drones to shots can now be computed by finding an assignment that minimizes the
total heuristic cost incurred. Recall from Definition \ref{def:shot-sequence} that $S_j^{t_a}$ denotes the shot $j$ that has been assigned to begin at time $t_a$, $A_i$ denotes the sequence of shots that have been assigned to drone $i$, and  $\mathcal{A} = \cup_i A_i$ is the set of all assigned shots and times. We
must search for an assignment $\mathcal{A}$ that minimizes the total sum of shot costs, while ensuring that
each drone has enough time to transition between the shots it has been assigned. 

The distance between shots is not deterministic, as it depends on the target's stochastic
trajectory. Moreover, even
if the drone cannot fully complete the transition between shots in the allotted time, its actual
position at the beginning of the next shot may be as good as the desired reference position
$\hat{\mathbf{x}}(t)$, especially if it can get close to $\hat{\mathbf{x}}$. Instead of hard
transition feasibility constraints, we introduce a cost function that penalizes transitions that may be dynamically infeasible. For a pair of shot assignments $S_j^{t_a}, S_k^{t_b}$, the transition cost $H_{trans}$ is defined as
\begin{equation*}
\begin{split}
H_{trans}(S_j^{t_a},& S_k^{t_b}) = \\&\mathbb{E}(\norm{[\max(0, d_{trans} - d_{max}), d_{trans}]^T}_{Q_{t}}),
\end{split}
\end{equation*}
where $d_{max}$ is the maximum distance the drone can travel between time $t_a$ and $t_b$ and
$d_{trans}$ is the distance required to transition between the end of shot $S_j^{t_a}$ and the
beginning of shot $S_k^{t_b}$.
The weighting $Q_t$ is set such that $H_{trans}$ incurs a very large penalty when the transition distance between shots is expected to be too long (i.e. $d_{trans} > d_{max}$), and a smaller penalty based on $d_{trans}$ which encourages assignments that result in the drones traveling shorter distances.

In addition to this cost term, we require that if one shot precedes another in the assignment
ordering, the first shot must end before the second begins. We define slightly different notation
for the heuristic cost $H_{shot}$ to refer to the cost of an assignment rather than the cost
of a shot. For assignment $A_i^j = S_k^{t_a}$, we notate the assignment cost as $H(A_i^j) = H^k(t_a)$. We want to solve the
following minimization:
\begin{subequations}
\label{eq:assignment_min}
\begin{equation}
\begin{split}
& \mathcal{A}^*  = \\ & \argmin_\mathcal{A} \sum_{i=1}^{n_d} \bigg( \sum_{j=1}^{|A_i|} H_{shot}(A_i^j)
    + \sum_{j=2}^{|A_i|} H_{trans}(A_i^{j-1}, A_i^j)\bigg)
    \label{eq:assignment_min:costs}
\end{split}
\end{equation}
\begin{alignat}{2}
\text{subject to} ~~     & \forall s \in S,~~s \in \mathcal{A}, \label{eq:assignment_min:all_shots}\\
& |S| = |\mathcal{A}|,\label{eq:assignment_min:no_extra}\\
& \forall S_j^t \in A,~~t \in \mathbf{t}_{sampled}^j,
\label{eq:assignment_min:sampled_times}\\
& \forall (S_j^{t_a}, S_{k}^{t_b}) \in A_i,~~t_b \geq t_a + \tau^j.
\label{eq:assignment_min:partial}
\end{alignat}
\end{subequations}

Constraints \eqref{eq:assignment_min:all_shots} and \eqref{eq:assignment_min:no_extra} ensure that
each shot appears exactly once in the ordering. Constraints \eqref{eq:assignment_min:sampled_times} and \eqref{eq:assignment_min:partial} force the chosen times to be from $\mathbf{t}_{sampled}$ and satisfy the temporal consistency previously discussed. 

The cost associated with \eqref{eq:assignment_min} acts like an upper bound on \eqref{eq:optimization-full}. For each shot, the assignment associates a time window, a drone, and a heuristic reference path that minimizes most of the constituent costs in $\mathbf{J}$ (the exception being $c_{scale}$). The cost is not truly an upper bound because $\hat{\mathbf{x}}$ may not result in a feasible trajectory, but it provides intuition for why we expect minimizing the cost in \eqref{eq:assignment_min} will provide a good reference trajectory for minimizing the final videography cost from \eqref{eq:optimization-full}. \vspace{-.1cm} 
\subsection{Constructing a Graph Formulation}

To find a minimizing solution to \eqref{eq:assignment_min}, we draw inspiration from minimum-cost flow assignment methods \cite{network_flows}. While this assignment problem cannot actually be solved with minimum-cost flow solvers, encoding it in a graph yields a straightforward interpretation as a constrained minimum-cost path problem in a Partitioned Directed Acyclic Graph (PDAG), which can be solved with an ILP.

\begin{define}[PDAG]\label{def:pdag}
A Partitioned Directed Acyclic Graph (PDAG) is a Directed Acyclic Graph $\mathcal{G}$, such that each vertex is assigned one of $k$ labels.
\end{define}
For this purposes of illustration, we will refer to these labels of the PDAG by $k$ different colors. We now define Problem \ref{prog:pdag_paths} on finding a set of minimum-cost paths through the PDAG. 
\begin{problem}[PDAG-Minimum-Paths]\label{prog:pdag_paths}
Consider a weighted PDAG, $\mathcal{G}$. Given a maximum number of paths $m$, PDAG-Minimum-Paths(m,
$\mathcal{G})$ is the problem of finding the lowest-weight set of paths $\mathcal{P}$ such that $|\mathcal{P}| \leq m$ and exactly one vertex of each color is on some path or determining that there is no such satisfying $\mathcal{P}$.
\end{problem}
\vspace{-.1cm}
Figure \ref{fig:optimization_constraints_full} demonstrates how a solution to a PDAG-Minimum-Paths problem with $m=2$ encodes the shot assignment problem. By design, a path must visit each of the graph partitions, which means that it passes through every color label of the graph. We now discuss how to construct a PDAG that encodes the constraints of  \eqref{eq:assignment_min}.

Note that \eqref{eq:assignment_min:partial} implies that the assignment of shots must
satisfy a strict partial ordering that ensures the sequence is monotonically increasing in time and
has no shot overlaps assigned to a single drone. 
We use this partial order to generate a DAG $\mathcal{G}_{shot}$, where
vertices represent each $S_i^t$, for $t \in \mathbf{t}^i_{sampled}$. This selection of vertices enforces \eqref{eq:assignment_min:sampled_times}. We choose the color label of each vertex $S_i^t \in \mathcal{G}_{shot}$ to be its shot index $i$. We also add a vertex in
$\mathcal{G}_{shot}$ for each of $n_d$ drones that can be assigned a shot. These drone vertices are
connected to all existing vertices in $\mathcal{G}_{shot}$ and each have a unique color.
The directed edge between $S_j^{t_a}$ and $S_k^{t_b}$ (if it exists) is denoted
$\mathcal{E}_{S_j^{t_a}S_k^{t_b}}$ with cost given by 
\begin{equation}\label{eq:gshot_cost_1}
H(\mathcal{E}_{S_j^{t_a}S_k^{t_b}}) = H_{shot}(S_k^{t_b}) +
H_{trans}(S_j^{t_a}, S_k^{t_b}).
\end{equation}
The cost for edges connected to the drone nodes is denoted
\begin{equation}\label{eq:gshot_cost_2}
H(\mathcal{E}_{D_i, S_k^{t_b}}) = H_{shot}(S_k^{t_b}).
\end{equation}
We now present Proposition \ref{prop:pdag-equiv}, which connects $\mathcal{G}_{shot}$ to  \eqref{eq:assignment_min}.
\begin{prop}\label{prop:pdag-equiv}
~\newline
Solutions to PDAG-Minimum-Paths($n_d, G_{shot}$) also minimize \eqref{eq:assignment_min}.
\end{prop}
\begin{proof}
We will show that the constraints and costs in \eqref{eq:assignment_min} are equivalent to the constraints and costs of solutions to PDAG-Minimum-Paths($n_d, G_{shot}$) (PMP). Let each shot $s \in \mathcal{S}$ correspond to a color in $G_{shot}$. If we consider the set of assigned shots $\mathcal{A}$ as corresponding to the set of vertices in $G_{shot}$ that are on paths selected by PMP, then enforcing each shot $s \in S$ to appear in the assignment A as in \eqref{eq:assignment_min:all_shots} is equivalent to enforcing that every color appears on the solution to PMP. Enforcing $|\mathcal{S}| = |\mathcal{A}|$ as in \eqref{eq:assignment_min:no_extra} is then equivalent to restricting each color to appear at most once on the PMP solution. The vertices of $G_{shot}$ are chosen by construction to respect constraint \eqref{eq:assignment_min:sampled_times}. The connectivity of $G_{shot}$ was defined by the partial ordering implied by \eqref{eq:assignment_min:partial}, so every solution to PMP respect this ordering constraint. Each path in the solution of PMP corresponds to the assignment ordering $A_i$ for a drone. The cost of the path consists of the shot cost connected to each vertex and the transition cost connected to each edge, as in \eqref{eq:gshot_cost_1} and \eqref{eq:gshot_cost_2}, the same expressions used to express the costs of assignments $A_i$ in \eqref{eq:assignment_min:costs}. As the constraints and costs are equivalent, the solution to PDAG-Minimum-Paths($n_d, G_{shot}$) also minimizes \eqref{eq:assignment_min}. 
\end{proof}

\subsection{Solving PDAG-Minimum-Paths}\label{sec:ilp}

We now formulate PDAG-Minimum-Paths as an ILP. Consider a weighted PDAG $\mathcal{G}
= (\mathcal{V}, \mathcal{E})$. To solve PDAG-Minimum-Paths($n_d, \mathcal{G}$), we first construct a
modified graph $\mathcal{G'} = (\mathcal{V}', \mathcal{E}')$ which is the same as $\mathcal{G}$, but it
contains an extra $n_d$ vertices. Each of the extra vertices is connected to all of the original
vertices, and they each have a unique color. These additional vertices are denoted
$\hat{\mathcal{V}}$.\footnote{Note that in the case of the drone assignment problem, the ``drone''
vertices that were added can be considered as $\hat{\mathcal{V}}$.}

We construct an ILP that contains a variable $x_i$ for each edge $e_i \in \mathcal{E}'$. The cost of each variable $C_x$ is the same as the cost of its corresponding edge.
Let $\hat{\mathcal{X}}$ denote the set of variables corresponding to edges connected to
$\hat{\mathcal{V}}$. Let $\mathcal{X}_{*,i}$ denote the set of variables corresponding to edges
incident on $V_i$. Let $\mathcal{X}_{i,*}$ denote the set of variables corresponding to edges
directing away from $V_i$. Let $\mathcal{X}^k$ denote the set of variables corresponding to edges
incident on color $k$. The ILP solving PDAG-Minimum-Paths($D, \mathcal{G}$) can be defined as:
\begin{subequations}
\label{eq:ip_formulation}
\begin{alignat}{2}
    \mathcal{X}^* = &\argmin_{\mathcal{X}} \sum_{x \in \mathcal{X}} xC_x\\
\text{s.t.} ~~ &\forall x \in \mathcal{X}, x \in \{0, 1\},
\label{eq:ip_formulation:binary}\\
&\forall k,\  \sum\nolimits_{i = 1}^{|\mathcal{X}^k|} \sum\nolimits_{x \in \mathcal{X}^k_{*, i}} x = 1,
\label{eq:ip_formulation:incoming_constraint}\\
&\forall i,\  \sum\nolimits_{x \in \mathcal{X}_{i,*}} x \leq \sum\nolimits_{x \in \mathcal{X}_{*,i}} x
\label{eq:ip_formulation:consistency_constraint}\\
&\forall i,\  \sum\nolimits_{x \in \hat{\mathcal{X}}_{i,*}} x \leq 1.
\label{eq:ip_formulation:drone_constraint}
\end{alignat}
\end{subequations}

The edges corresponding to variables with a value of one in $\mathcal{X}^* \setminus \hat{\mathcal{X}}$ are considered to be the set of edges in the solution to PDAG-Minimum-Paths. If there is no satisfying assignment for $\mathcal{X}^*$, then there is no solution to PDAG-Minimum-Paths.

\begin{prop} The 
Integer Linear Program in \eqref{eq:ip_formulation} generates a solution PDAG-Minimum-Paths, which minimizes \eqref{eq:assignment_min}.
\end{prop}
\begin{proof}
Constraint \eqref{eq:ip_formulation:incoming_constraint} ensures that the total number of edges
incident on a color set is equal to one. This satisfies that all colors are part of some path.
Constraints \eqref{eq:ip_formulation:consistency_constraint} and
\eqref{eq:ip_formulation:drone_constraint} restrict the solution set to contain at most $D$ paths, by
limiting the number of selected edges attached to $\hat{\mathcal{V}}$ to $D$ and requiring that all
other edges start at a node with an incident edge. Thus, every solution of \eqref{eq:ip_formulation} encodes a valid set of paths in PDAG-Minimum-Paths. By the construction of \eqref{eq:ip_formulation}, every PDAG-Minimum-Paths problem can be represented by the ILP. Thus, \eqref{eq:ip_formulation} solves PDAG-Minimum-Paths, and by Proposition \ref{prop:pdag-equiv}, it minimizes \eqref{eq:assignment_min} as well.
\end{proof}
The size of the ILP necessary to solve the drone assignment depends on the number of shots $|\mathcal{S}|$, the number of samples we choose per shot $p$, and the number of drones, $n_d$. The ILP has variables corresponding to transitions between the $O(|\mathcal{S}|p)$ sampled times of the set of shots. For each drone, there is also a variable associated with each sampled time. This results in a ILP with $O((|\mathcal{S}|p)^2 + |\mathcal{S}|pn_d)$ variables. There area total of $2|\mathcal{S}|p + |\mathcal{S}|$ constraints. We find in practice that the number of shots we can define is limited by the horizon over which we can make predictions of the target's motion rather than by computational burden of the assignment algorithm.

The full high-level shot assignment pipeline is summarized in Algorithm \ref{alg:videography}.

\begin{algorithm}
\caption{Videography Assignment}\label{alg:videography}
\begin{algorithmic}[1]
\Procedure{AssignShots}{$\mathcal{S}, \mathbb{Y}, \mathcal{O}$}
\For{$i = 1 : |\mathcal{S}|$} \label{alg:videography:loop1}
\State $H_{\text{shot}}^i \gets \text{CalculateShotHeuristic}(\mathbb{Y},\mathcal{O}, S_i)$
\label{alg:videography:cshot}
\State $\mathbf{t}_{\text{sampled}}^i \gets \text{ImportanceSample}(H_{\text{shot}}^i)$
\label{alg:videography:csample}
\EndFor
\For {\textbf{each }$ (S_i^{t_a}, S_j^{t_b}) \in$ PotentialShotTransitions} \label{alg:videography:loop2}
\State $H_{\text{trans}}(S_i^{t_a}, S_j^{t_b}) \gets \mathbb{E}(\text{TransitCost}(S_i^{t_a},
S_k^{t_b}, \mathbb{Y}))$\label{alg:videography:ctrans}
\EndFor
\State $\mathcal{A}^* \gets \text{IntegerProgram}(t_{\text{sample}}, H_{\text{sample}}, H_{\text{trans}}, \mathcal{S})$
\label{alg:videography:ip}
\State \textbf{return} $\mathcal{A}^*$
\EndProcedure
\end{algorithmic}
\end{algorithm}
\vspace{-.4cm}
\section{ONLINE VIEWPOINT OPTIMIZATION}\label{sec:mpc}
The assignment $\mathcal{A}^*$ from \eqref{eq:assignment_min} is used to guide the online MPC. For each shot assignment $S_j^{t_a} \in \mathcal{A}_i^*$, drone $i$ will attempt to minimize the viewpoint cost for shot $j$ in the range $[t_a, t_a + \tau^j]$. In the intermediate time between successive assignments $S_j^{t_a}$ and $S_k^{t_b}$, the MPC follows a reference path that leads it from the expected end of one shot to the expected beginning of another, $\tilde{\vx}(t_a + \tau^j)$ to $\tilde{\vx}(t_b)$. We precompute a Probabilistic Roadmap (PRM)\cite{kavraki1996probabilistic} of free space in the environment, and use it to compute the reference path between $\tilde{\mathbf{x}}(t_a + \tau^j)$ and $\tilde{\mathbf{x}}(t_b)$.
During each shot period, the MPC locally optimizes the viewpoint costs. The MPC solves for drone trajectories and gimbal controls in real-time separately for each drone. The rest of the
MPC formulation is materially the same as the method presented in \cite{nageli_real-time_2017}, but we summarize it here for completeness. For speed considerations and analysis of the nonlinear videography MPC performance, we refer to \cite{nageli_real-time_2017}.
\vspace{-.1cm}
\subsection{Drone and Camera Model}
We assume the drone's motion model
follows some nonlinear differential equation $\dot{\mathbf{x}} = f(\vx, \vu)$, with $\vx(0) = \vx_0$ and 
 $\vx^k \in \mathcal{X} \subset \mathbb{R}^{n_x}$ denotes the state of the drone and $\vu^k \in
\mathcal{U} \subset \mathbb{R}^{n_u}$ the control inputs at time step $k$. $\mathcal{X}$ and
$\mathcal{U}$ are the admissible state space and control space, respectively. While $\mathbf{x}$ has been used in previous sections to denote only position, here it represents the full state. $\vx_0$ is the initial
state of the drone. Our experiments use the Parrot Bebop 2 quadrotor with its integrated front camera and
electronic gimbal. It accepts desired roll and pitch angles, yaw rate,
vertical velocity, and gimbal angles as input.
\vspace{-.1cm}
\subsection{MPC Formulation}

For each drone in the team, we formulate a receding horizon constrained optimization problem with
$N$ time steps and planning horizon $N\Delta t$, where $\Delta t$ is the sampling time. The
optimization seeks to minimize the weighted squared norm of a cost vector $\norm{\hat{\mathbf J}}^2_{Q}$.
Recall that $\mathbf{J}_i^j$ denotes the instantaneous viewpoint cost incurred by drone $i$ capturing shot $j$. As desired viewpoint and drone are unambiguous in the MPC formulation, we drop the sub- and superscripts in this section.
We augment $\mathbf{J}$ to include the control inputs, distance from reference state $\vx_{ref}$, and collision avoidance slack variables $\mathbf{S}$:
\begin{equation*}
\hat{\mathbf{J}}^k = [\mathbf{J}^T, \vu^T, \norm{\vx_{ref} - \vx}, \mathbf{S}^T]^T|_{t=k}.
\end{equation*}

A matrix of weightings $Q$ defines the relative importance of each cost function. When the drone is assigned to take a shot, the weighting matrix does not penalize reference position error. When the drone transitions between shots, the relative viewpoint and distance are not penalized, although the image position cost $c_{image}$ is kept. This keeps the drones pointing toward the target while repositioning for the next shot.

We use a third order
collocation method that implicitly represents the drone's trajectory as a third order spline and the
control inputs as piecewise linear functions of time as described in \cite{underactuated}. The dynamics constraints are applied at the
collocation points $t_{c,n}$, and the drone state. The full MPC optimization can be written as:
\begin{subequations}\label{eq:mpc}
	\begin{alignat}{2}
		\min\limits_{\vx^{1:N}, \vu^{1:N}}  ~~        
		& \sum_{k=1}^{N-1} \norm{\hat{\mathbf J}^k}^2_Q + \norm{\hat{\mathbf J}^N}^2_{Q_f}\\
		\text{s.t.}	~~	        & \vx^1 = \vx(0), \\
        & \dot{\vx}(t_{c,n}) = f(\vx(t_{c,n}), \vu(t_{c,n})),\\
        & \vx^k \in \mathcal{X},\\
        & \vu \in \mathcal{U}, \\  
		& \forall k\in \{1,\dots,N\}.
	\end{alignat}
\end{subequations}

At each time step, the drone solves the formulated nonlinear constrained optimization problem online to
generate a local trajectory and executes the first time step controls in a receding
horizon fashion. The optimization is performed using CasADi \cite{Andersson2019} for automatic differentiation and IPOPT \cite{wachter_implementation_2006} for optimization. Note that the optimization is not necessarily solved to completion --- the current iterate after 40 ms of computation is used to select the next control output.

In order to estimate the costs at each timestep in the horizon, the drone must have
an estimate for the future trajectories of dynamic obstacles and the videography target. Our prediction model assumes constant linear and angular velocity in the target's body frame, based on current estimated velocity. For our experiments the velocity estimate is provided by a motion capture system. 

We implement the MPC for multiple drones in an asynchronous manner: the MPC is solved independently
for each drone, and the expected trajectory each drone generates is used to compute collision
avoidance by future iterations of the other drones' MPC as described in detail in \cite{Zhu2019RAL}.

\section{EXPERIMENTAL RESULTS\label{sec:experiments}}

We validated the tracking system with a pair of Bebop 2 drones tracking a human-driven RC racecar\footnote{\url{https://mit-racecar.github.io/}}. In order to generate the predictive target distribution necessary to calculate expected shot costs, the racecar was driven around two obstacles in a loosely defined pattern for ten warmup laps. Figure \ref{fig:experiment_track} shows the paths from these ten runs and their mean. Note that the car's speed varies within each run and comes to a stop at several locations. The warmup runs were used to model a multivariate Gaussian distribution over the car's trajectory. Nine desired shots were specified for the drone, varying in length from four to six seconds and all constrained to occur within the ninety second duration of the car's lap. The image position parameters $C_x$ and $C_y$ were set the center the target in the frame in all shots. The settings for the other parameters are shown in Figure \ref{fig:experimental_plot}. The high-level shot planner assigned each drone a sequence of shots. The cost $H_{shot}(t)$ for heuristic position $\hat{\mathbf{x}}(t)$ for each shot is shown in Figure \ref{fig:experiment_heuristic_costs}. The planner is able to find an assignment of shots to the drones such that the expected viewpoint cost for each shot is quite low. We note that the planner assigns overlapping shots to the two drones resulting in a sequence of shots that would have been infeasible with a single drone.

The resulting assignments were tracked over a series of ten trials in which the racecar was manually driven along the same route as the warmup laps. Figure \ref{fig:experiment_track} shows the trajectories for the drones during one of the runs, in addition to the car's nominal path and two obstacles. The figure also demonstrates that the planning algorithm's penalty for transition distance between shots $H_{trans}$ leads to the assigned shots being spatially segmented between the drones, reducing the total distance each has to travel.

Figure \ref{fig:experimental_plot} illustrates the relative viewpoint achieved by each
drone relative to the desired parameters during each shot. During each assigned shot window (denoted in green), the viewpoint parameters (in blue) are close to their desired values (in orange). The high-level planning heuristic and reference trajectory places the drones near the desired viewpoints, so in many cases the desired viewpoint is reached before the shot even starts. The drones are able to achieve lower-deviation trajectories during some shots compared to others. The higher-deviation shots occur when the desired drone position is obstructed (because of a static obstacle or other drone), or if the trajectory necessary to follow the desired position is dynamically infeasible (most commonly when the car is turning quickly). Even in cases where there is significant deviation from the desired viewpoint, such as the distance in Drone 1's third shot (at about T = 25s in \eqref{fig:experimental_plot:d1dist}), we note that the other two parameters have low deviation and the drone maintains an unoccluded view of the target as shown in the accompanying video. Instances where the observed viewpoint deviations are greater than expected from the heuristic costs in Figure \ref{fig:experiment_heuristic_costs} are mostly due to the drones acting more conservatively when close to obstacles than expected by the planner.

\begin{figure}
    \vspace*{-.2cm}
    \begin{subfigure}{\columnwidth}
    \centering
    \includegraphics[width=.8\columnwidth]{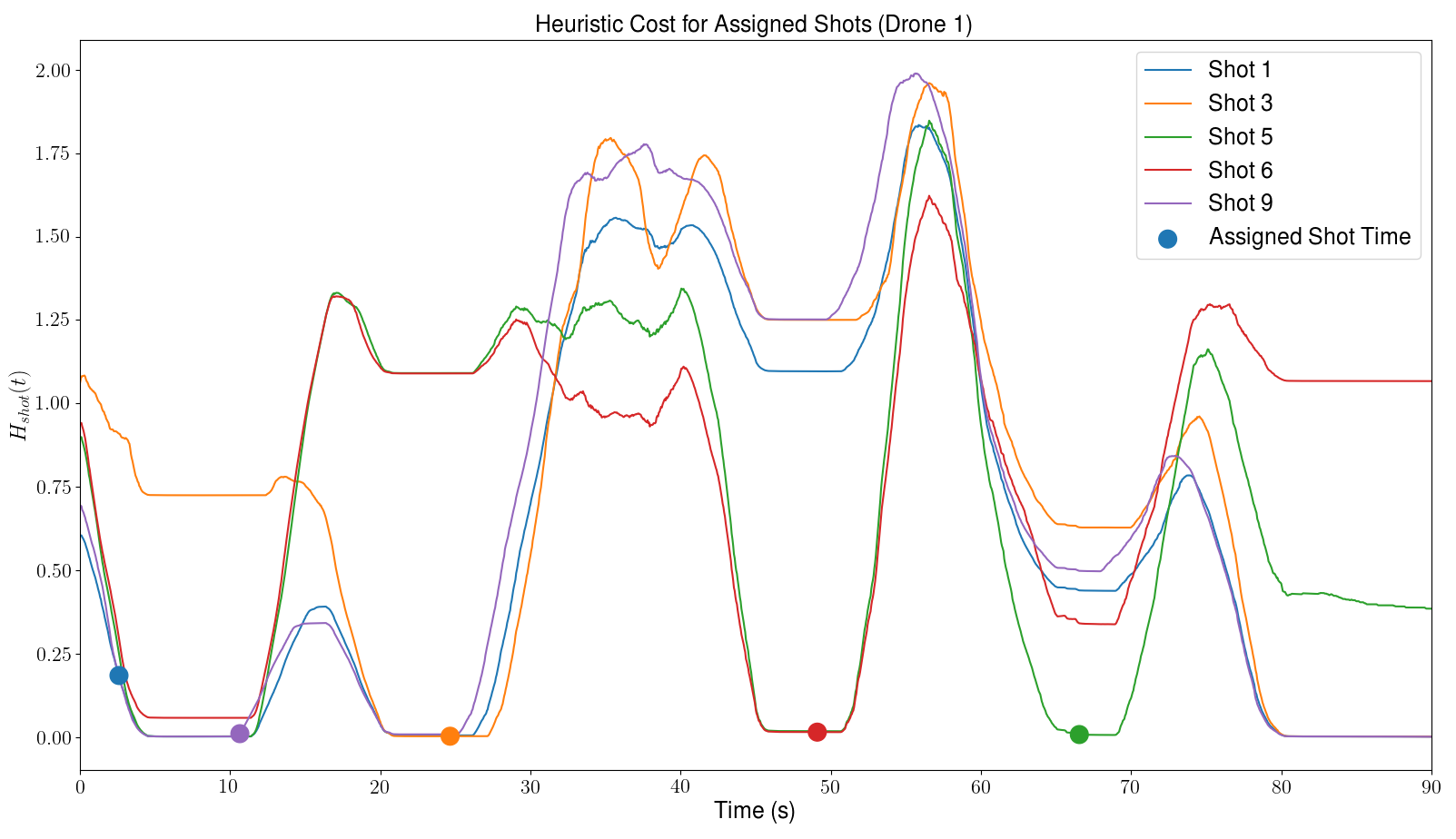}
    \caption{Expected cost $H_{shot}(t)$ of path heuristic $\hat{\mathbf{x}}(t)$ for each of the five shots ultimately assigned to Drone 1.}
    \label{fig:experiment_heuristic_costs:1}
    \end{subfigure}
    \begin{subfigure}{\columnwidth}
    \centering
    \includegraphics[width=.8\columnwidth]{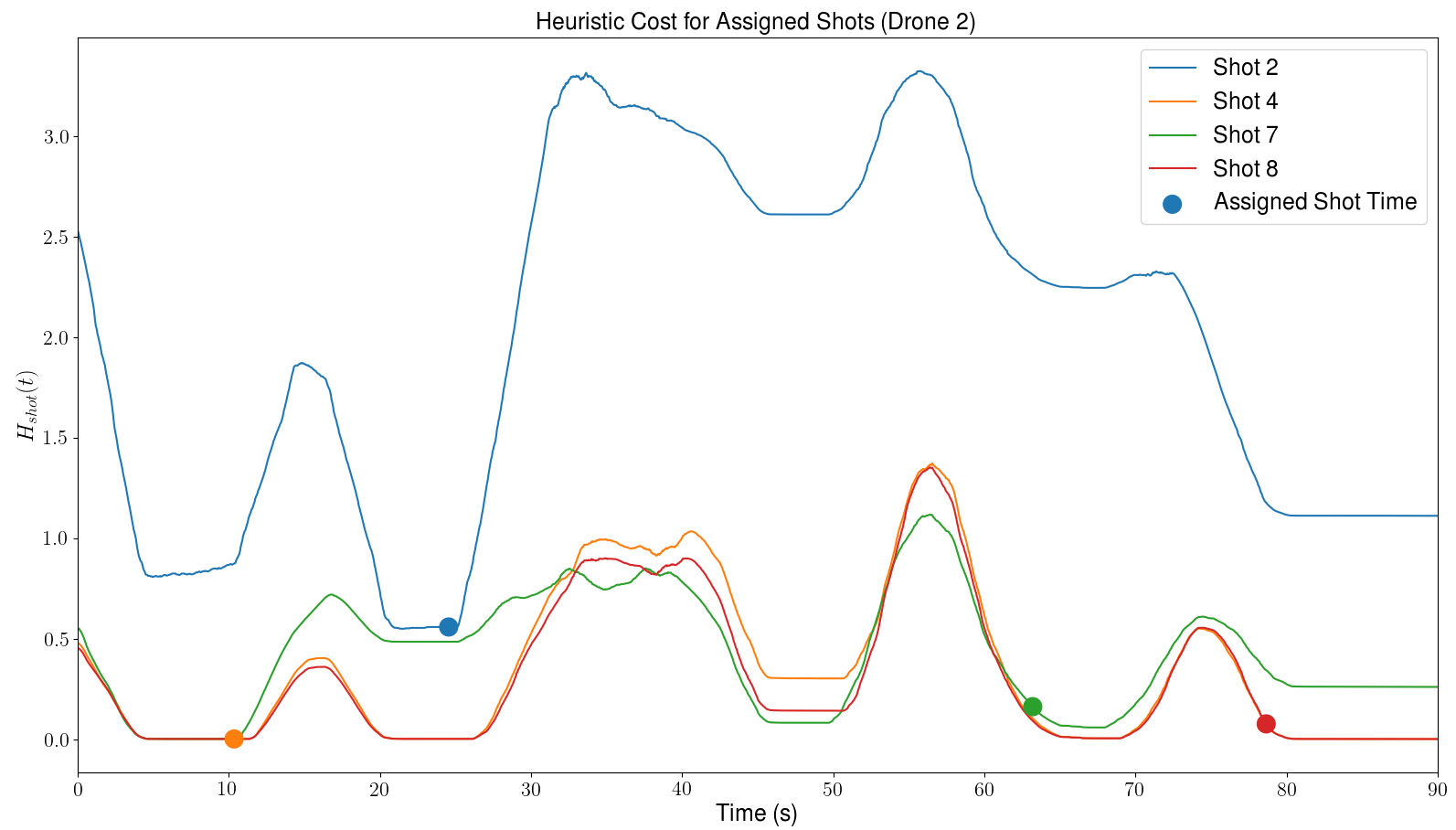}
    \caption{Expected cost $H_{shot}(t)$ of path heuristic $\hat{\mathbf{x}}(t)$ for each of the four shots ultimately assigned to Drone 2.}
    \label{fig:experiment_heuristic_costs:2}
    \end{subfigure}
    \caption{Cost associated with heuristic path over time for each desired shot. The solid circles denote the time that the high level planner determined was the lowest expected cost.}
    \label{fig:experiment_heuristic_costs}
\end{figure}
\begin{figure}
\vspace*{-1.5cm}
    \begin{subfigure}{0.47\columnwidth}
        \centering
        \includegraphics[width=\textwidth]{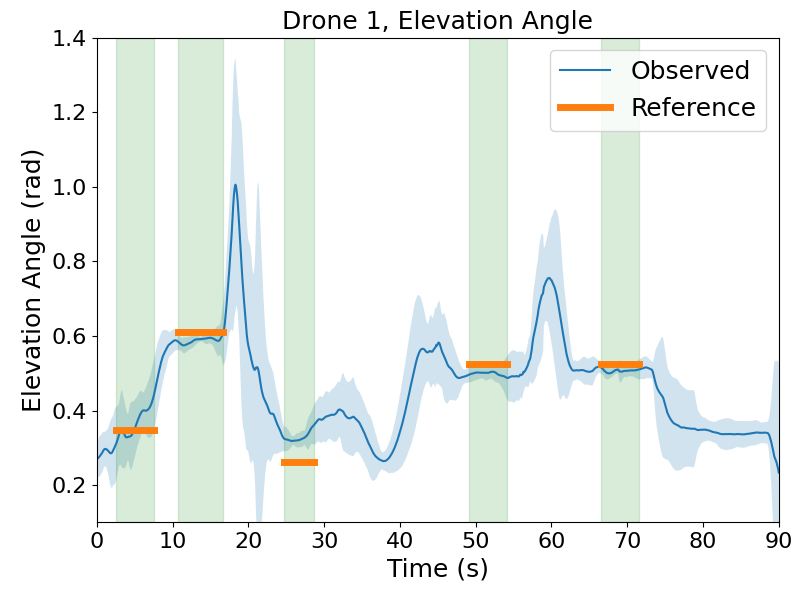}
        \caption{Drone 1 Elevation}
    \end{subfigure} 
        \begin{subfigure}{0.47\columnwidth}
        \centering
        \includegraphics[width=\textwidth]{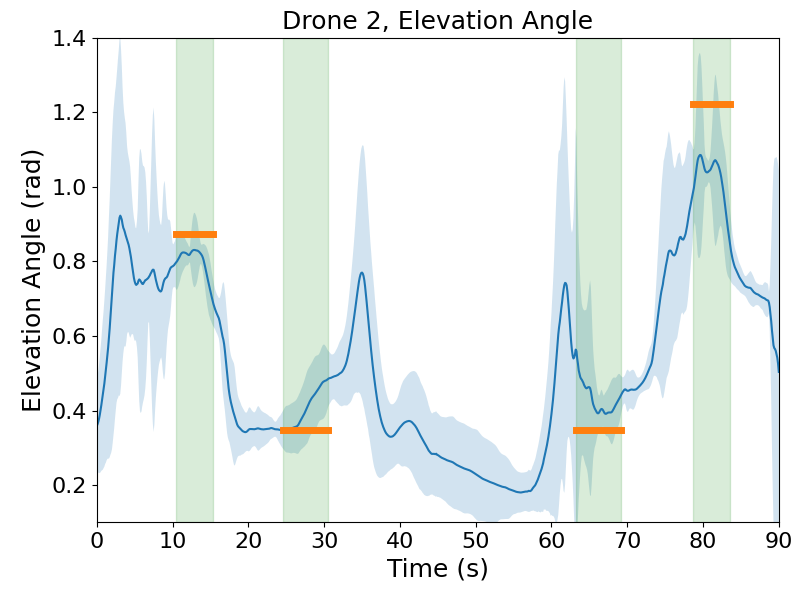}
        \caption{Drone 2 Elevation}
    \end{subfigure} 
        \begin{subfigure}{0.47\columnwidth}
        \centering
        \includegraphics[width=\textwidth]{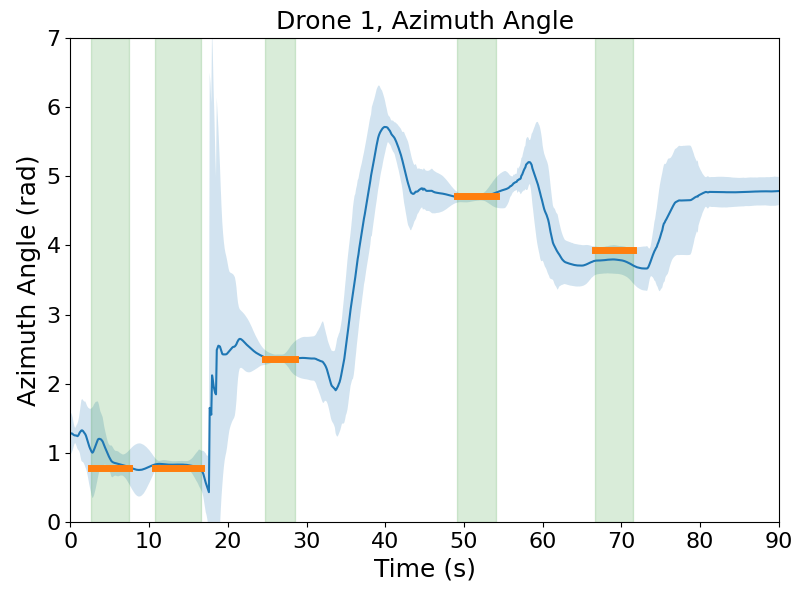}
        \caption{Drone 1 Azimuth}
    \end{subfigure} 
        \begin{subfigure}{0.47\columnwidth}
        \centering
        \includegraphics[width=\textwidth]{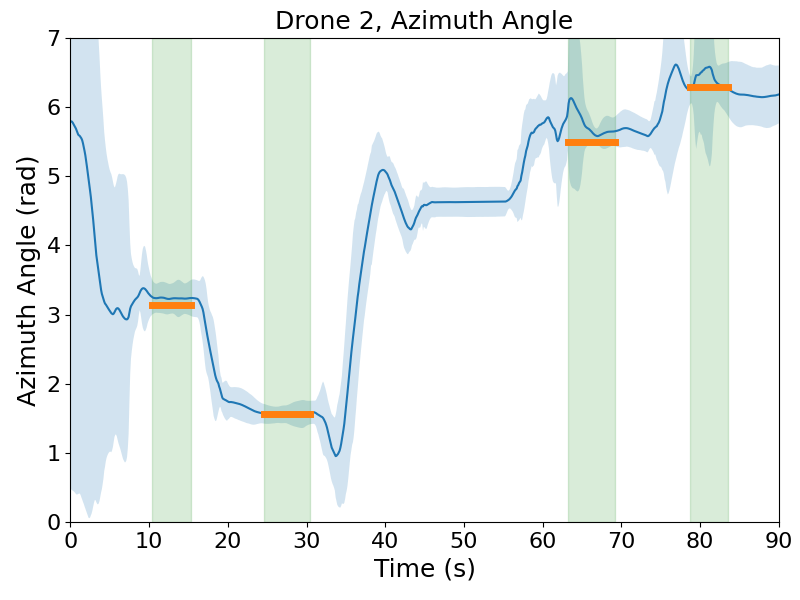}
        \caption{Drone 2 Azimuth}
    \end{subfigure} 
        \begin{subfigure}{0.47\columnwidth}
        \centering
        \includegraphics[width=\textwidth]{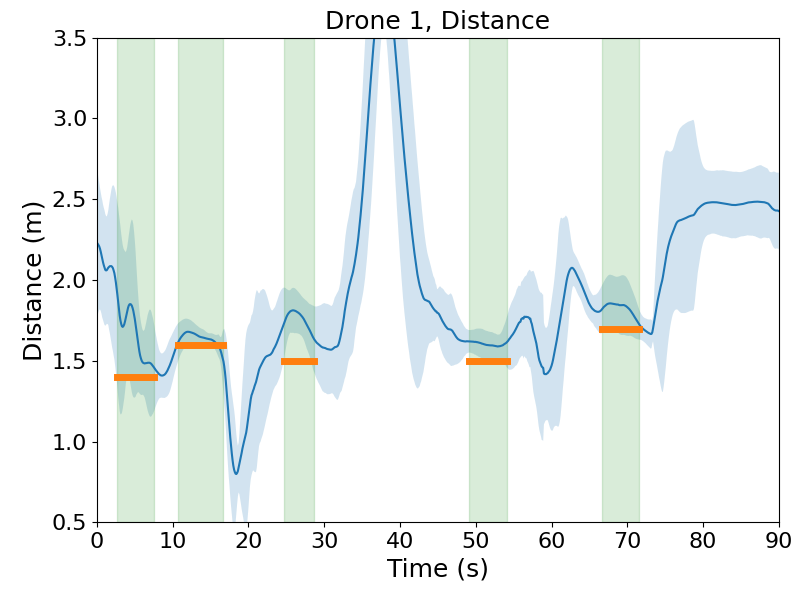}
        \caption{\label{fig:experimental_plot:d1dist}Drone 1 Distance}
    \end{subfigure} 
        \begin{subfigure}{0.47\columnwidth}
        \centering
        \includegraphics[width=\textwidth]{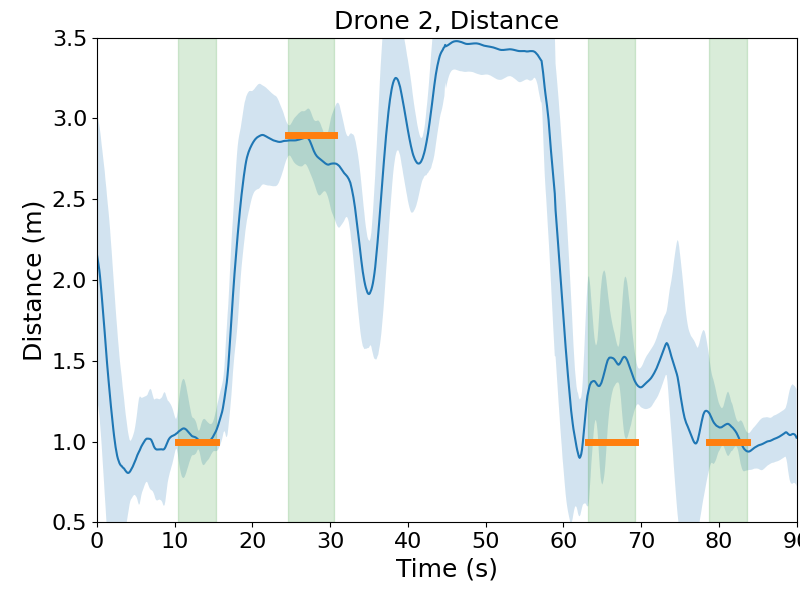}
        \caption{Drone 2 Distance}
     \end{subfigure} 
 \caption{Performance of the two drones. Active shots are shaded in green, and orange bars represent the reference. The solid blue line represents the mean value over the 10 trials, and the transparent blue denotes one standard deviation of the data from the 10 trials. The drones converge to the reference values while in a shot.}
\label{fig:experimental_plot}
\vspace{-1cm}
\end{figure}

\begin{figure}
    \centering
    \includegraphics[width=.8\columnwidth]{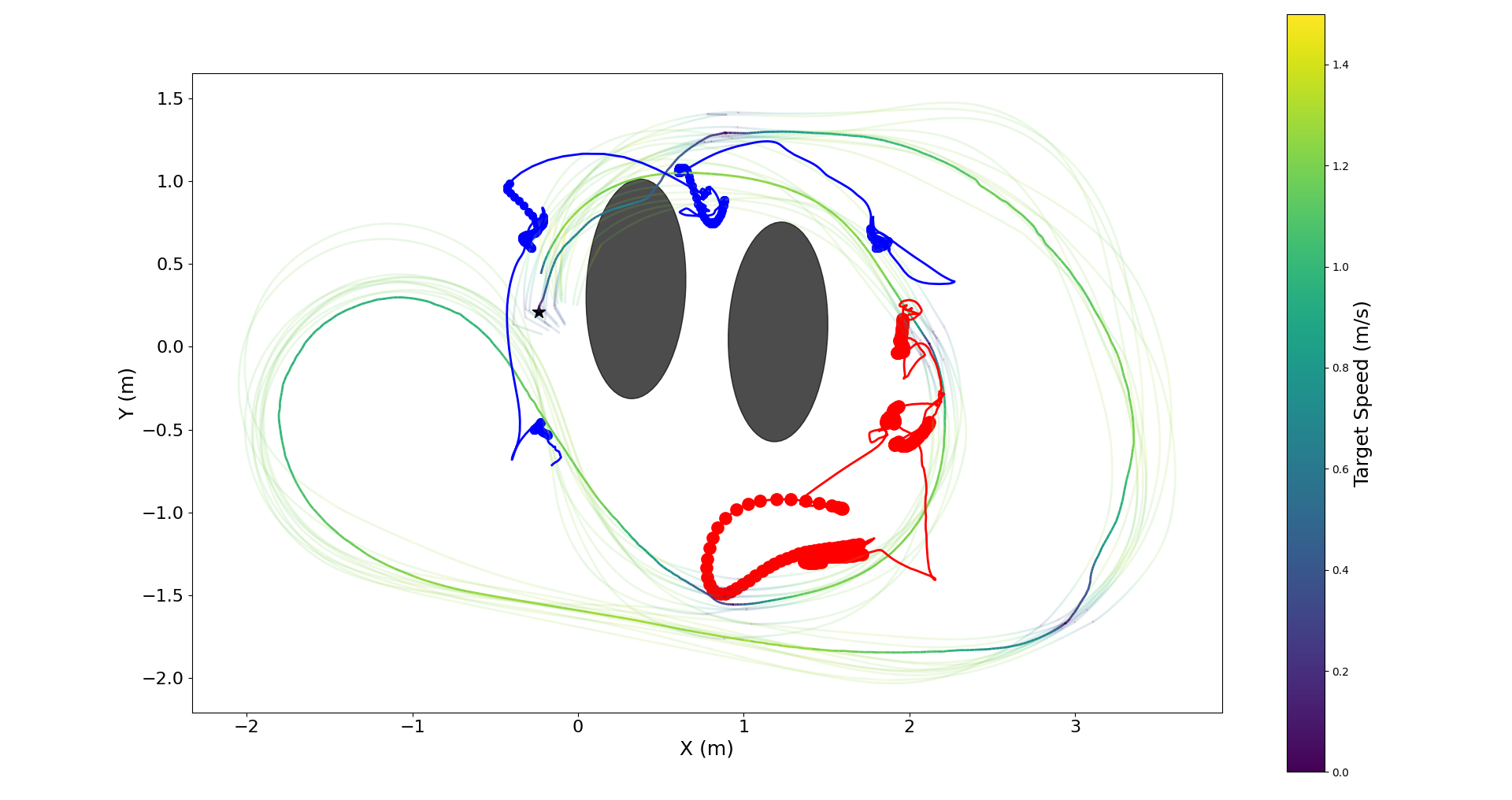}
    \caption{The predicted mean of the racecar trajectory (opaque blue/yellow curve) and its training runs (transparent blue/yellow curves). Car speed shown on color axis. The drone trajectories for one run are shown in red and blue. Circular markers denote times when the drones were taking a shot.}
    \label{fig:experiment_track}
    \vspace{-.5cm}
\end{figure}

The performance bottleneck for the high-level assignment algorithm is solving the ILP. The size of the ILP depends on the number of shots and the number of sampled times per shot. Our scenario contained nine shots, with twenty times sampled for each shot. The resulting ILP solved in under thirty seconds on a laptop with a 2.6 GHz Intel Core i7-9750H 6-Core processor and 16 GB RAM. Different balances between optimality and solution speed can be achieved by tuning the number of sampled times. We note that while the ILP does not run in real time, it would have had time to run at least once during the scenario to reassign shots based on updated predictions of the target trajectory. 
\vspace{-.2cm}
\section{CONCLUSIONS}\label{sec:conclusions}

This paper has demonstrated a practical algorithm for assigning a team of drones to capture a series
of desired shots and locally optimizing the drone trajectories to ensure each shot is captured as
well as possible. Our experiments have shown that the planning and control pipeline works on
physical systems in the presence of obstacles and uncertainty over the videography target's
trajectory. Future work will explore an increased uncertainty of the target’s
trajectory, as well as multi-target tracking. 

\addtolength{\textheight}{-.1cm}   




\bibliographystyle{IEEEtran}
\bibliography{IEEEabrv,references}

\end{document}